\documentclass[letterpaper]{article} 
\usepackage{aaai23}  

\usepackage{times}  
\usepackage{helvet}  
\usepackage{courier}  
\usepackage[hyphens]{url}  
\usepackage{graphicx} 
\usepackage{natbib}  
\usepackage{caption} 
\usepackage[utf8]{inputenc} 

\usepackage{booktabs}       
\usepackage{amsfonts}       
\usepackage{nicefrac}       
\usepackage{microtype}      
\usepackage{xcolor}         

\usepackage{graphicx}
\usepackage{epstopdf}
\usepackage{mathtools}
\usepackage{amsthm, amssymb}
\usepackage[ruled, vlined, linesnumbered]{algorithm2e}
\usepackage{caption, subcaption}
\usepackage{enumitem}
\usepackage{color}
\usepackage{comment}

\newcommand{\abs}[1]{\left| #1 \right|}

\newtheorem{theorem}{Theorem}
\newtheorem{lemma}[theorem]{Lemma}
\newtheorem{claim}[theorem]{Claim}
\newtheorem{corollary}[theorem]{Corollary}

\newtheorem{remark}[theorem]{Remark}

\DeclareMathOperator{\bE}{{\mathop{\mathbb{E}}}}
\newcommand{\D}{\mathcal{D}}
\newcommand{\A}{\mathcal{A}}
\newcommand{\E}{\mathcal{E}}
\newcommand{\G}{\mathtt{Top}(I)}

\newcommand{\DISJ}{{\em DISJ}}

\newcommand{\KZZ}{{\tt KZZ}}
\newcommand{\Unif}{{\tt Uniform}}
\newcommand{\CIID}{{\tt CollabTopIID}}
\newcommand{\CNIID}{{\tt CollabTopNIID}}

\title{Communication-Efficient Collaborative Best Arm Identification}

\author{
	Nikolai Karpov, \ \  
	Qin Zhang
}
\affiliations{
	Department Of Computer Science \\
	
	Indiana University \\
	Bloomington, IN 47405, USA \\
	nkarpov@iu.edu, qzhangcs@indiana.edu
}

\begin{document}

\maketitle

\begin{abstract}
We investigate top-$m$ arm identification, a basic problem in bandit theory, in a multi-agent learning model in which agents collaborate to learn an objective function.  We are interested in 
designing collaborative learning algorithms that achieve maximum speedup (compared to single-agent learning algorithms) using minimum communication cost, as communication is frequently the bottleneck in multi-agent learning.  We give both algorithmic and impossibility results, and conduct a set of experiments to demonstrate the effectiveness of our algorithms. 
\end{abstract}

\section{Introduction}
\label{sec:intro}

As the scale of reinforcement learning continues to grow, multi-agent learning has become a necessity in many scenarios to speed up the learning process. In this paper, we explore {\em collaborative learning}, a multi-agent learning model introduced in~\citet{HKK+13} and~\citet{TZZ19} for studying parallel reinforcement learning under communication constraints.  Communication is frequently the bottleneck in multi-agent learning owing to network latency, energy consumption, mobile data usage, and privacy considerations.  For example, in the deep sea and outer space exploration, data transmission between agents (sensors, robots, etc.) is expensive since it consumes a lot of energy, which cannot be easily recharged. In the setting where communication makes use of a mobile network, data transmission directly adds to our bill. In scenarios where privacy is a concern, we also want to minimize the total amount of communication, since messages will leak information about local data which we do not want to share with other parties.
In this paper, we are interested in the following question:
\begin{quote}
\em How much communication is necessary and sufficient for scalable algorithms in the collaborative learning model?
\end{quote}

We study a basic problem in bandit theory: {\em top-$m$ arm identification} in multi-armed bandits (MAB); when $m = 1$, the problem is called {\em best arm identification}.   This problem has numerous applications, including medical tests, industrial engineering, evolutionary computation, and crowd-sourcing.  We study this problem in both the IID data setting in which agents learn from the same environment and the non-IID data setting in which agents learn from different environments.

\vspace{2mm}
\noindent{\bf The Collaborative Learning Model.\ \ }  We begin by introducing the collaborative learning model, in which $K$ agents communicate via a {\em coordinator} (central server) to learn a common objective function.  The learning process proceeds in rounds. During each round, each agent takes a sequence of actions (one at each time step) and receives a sequence of observations. Each action is taken based on the agent's previous actions and observations, messages received from other parties, and the randomness of the algorithm.  At the end of each round, each agent sends a message to the coordinator. After collecting messages and performing some computation, the coordinator sends a message back to each agent.  At the end of the last round, the coordinator outputs the answer.

The collaborative learning model is similar to the {\em federated learning} model,  which has attracted a lot of attention recently~\cite{corr/abs-1912-04977}.  The main difference is that in federated learning we aim at training machine learning models (e.g., neural networks) on multiple machines with local datasets, while the collaborative learning model is tailored for multi-agent reinforcement learning.

\vspace{2mm}
\noindent{\bf Top-$m$ Arm Identification in MAB.\ \ }  In this problem we have $n$ alternative arms; the $i$-th arm is associated with an unknown distribution $\D_i$ with support $[0,1]$.  Given a time budget $T$, we want to identify the top-$m$ arms/distributions with the largest means by a sequence of $T$ arm pulls (each pull takes a unit time). Upon each pull of the $i$-th arm, the agent observes
an i.i.d.\ sample from $\D_i$.  

In the collaborative learning model, agents pull arms in parallel in each round.  If multiple agents pull the same arm $i$ at a particular time step, then each agent gets an {\em independent} sample from distribution $\D_i$.  At the end of the last round, the coordinator needs to output the top-$m$ arms.

W.l.o.g., we assume that $n \ge 2m$, since otherwise, we can try to find the bottom $(n - m) < n/2$ arms, which is symmetric to identifying the top-$(n-m)$ arms.

\vspace{2mm}
\noindent{\bf Speedup in Collaborative Learning.\ \ }
We use the concept of {\em speedup} introduced in \citet{TZZ19} and \citet{KZZ20} to quantify the performance of multi-agent learning in the collaborative learning model.  For top-$m$ arm identification, given an input instance $I$ of $n$ arms with means $\mu_1, \ldots ,\mu_n$, define 
$$H^{\langle m \rangle} \triangleq H^{\langle m \rangle}(I) = \sum_{i \in [n]} \frac{1}{\Delta_i^2}, \quad \quad \text{where}$$ 
\begin{equation}
\label{eq:gap}
\Delta_i = \max\left\{\mu_i - \mu_{[m+1]}, \mu_{[m]} - \mu_{i}\right\},
\end{equation}
where $\mu_{[m]}$ stands for the $m$-th highest mean of the $n$ arms.  We will assume $\forall i \in [n], \Delta_i \neq 0$ (or, equivalently, $\mu_{[m]} \neq \mu_{[m+1]}$ so that the top-$m$ arms are unique), since otherwise $H^{\langle m \rangle} = \infty$.

It is known that there is a centralized algorithm for top-$m$ arm identification in MAB  that succeeds with probability at least $0.99$ given a time horizon $\tilde{O}(H^{\langle m \rangle})$~\cite{BWV13},\footnote{We use ` $\tilde{}$ ' on top of $O, \Omega, \Theta$ to hide logarithmic factors.}.  On the other hand, any centralized algorithm that achieves a success probability of at least $0.99$ requires $\Omega(H^{\langle m \rangle})$ time~\cite{CLQ17b}. 

Following \citet{KZZ20}, letting $\A$ be a collaborative algorithm that solves top-$m$ arm identification for each input instance $I$ with probability at least $0.99$ using time $T_{\A}(I)$, we define the speedup of $\A$ to be 
\begin{equation}
\label{eq:speedup}
\beta_\A = \min_{\text{all input instances } I} \frac{H^{\langle m \rangle}(I)}{T_{\A}(I)}.
\end{equation}
It has been observed that any collaborative algorithm $\A$ with $K$ agents can achieve a speedup of at most $\beta_\A \le K$~\cite{KZZ20}.  

\vspace{2mm}
\noindent{\bf Non-IID Data.\ \ } 
In the above definition for top-$m$ arm identification, when pulling the same arm, the $K$ agents sample from the same distribution. We call this setting ``learning with IID data''.  However, in many applications, such as channel selection in cognitive radio networks and item selection in recommendation systems, agents may interact with different environments~\cite{corr/abs-1912-04977}, or, ``learn with non-IID data''.  

For top-$m$ arm identification with non-IID data, by pulling the same arm, agents may sample from different distributions.  Let $\D_{i,k}$ be the distribution sampled by the $k$-th agent when pulling the $i$-th arm, and let $\mu_{i,k}$ be the mean of $\D_{i,k}$. We define the {\em global mean} of the $i$-th arm to be 
\begin{equation}
\label{eq:non-IID-mean}
 \mu_i \triangleq \frac{1}{K} \sum_{k \in [K]} \mu_{i,k}.
\end{equation}
At the end of the learning process, the coordinator needs to output the top-$m$ arms with the largest global means.

\vspace{2mm}
\noindent{\bf Our Results.\ \ }
In this paper, we are primarily interested in the case that a collaborative algorithm achieves almost full speedup, that is, $\beta = \tilde{\Omega}(K)$ where $K$ is the number of agents. We try to pinpoint the minimum amount of communication needed in the system for achieve such a speedup. For convenience, when talking about communication upper bounds, we assume each numerical value can be stored in one {\em word}.  While for lower bounds, we state the communication cost in terms of {\em bits}.  The difference between these two measurements is just a logarithmic factor.

In the IID data setting, we show that there is a collaborative algorithm for top-$m$ arm identification achieving a speedup of $\tilde{\Omega}(K)$ using $\tilde{O}(K+m)$ words of communication (Corollary~\ref{cor:full-speedup}).\footnote{All logarithmic factors hidden in ` $\tilde{}$ ' will be spelled out in the concrete theorems in this paper.} On the other hand, $\tilde{\Omega}(K+m)$ bits of communication is necessary to achieve $\tilde{\Omega}(K)$ speedup (Corollary~\ref{cor:iid-lb}).  

In the non-IID data setting, there is a collaborative algorithm for top-$m$ arm identification achieving a speedup of $\tilde{\Omega}(K)$ using $\tilde{O}(K n)$ words of communication (Theorem~\ref{thm:noniid-ub} and Corollary~\ref{cor:full-speedup-2}).  We also show that ${\Omega}(K n)$ bits of communication is necessary only to output the top-$1$ arm correctly, {\em regardless the amount of speedup} (Theorem~\ref{thm:non-iid-lb}).

The above results give a strong separation between the IID data setting and the non-IID data setting in terms of communication cost.  

In Section~\ref{sec:exp}, we have conducted a set of experiments which demonstrate the effectiveness of our algorithms.

\subsection{Related Work}

Best and top-$m$ arm identification in MAB have been studied extensively and thoroughly in the centralized model where there is a single agent;
(almost) tight bounds have been obtained for both problems~\cite{ABM10, CLQ17, BWV13, CLQ17b}.  
The two problems have also been investigated in the collaborative learning model~\cite{HKK+13,TZZ19,KZZ20}. However, these works focus on the {\em number of rounds of communication} instead of the actual communication cost (measured by the number of words exchanged in the system); they allow agents to communicate any amount of information at the end of each round.  The algorithms for top-$m$ arm identification in \citet{KZZ20} need $\tilde{\Omega}(nK)$ communication to achieve a $\tilde{\Omega}(K)$ speedup even in the IID data setting.   

On minimizing the communication cost in the collaborative learning model,~\citet{WHCW20} studied {\em regret minimization} in MAB in the IID data setting. Recently, \citet{SS21} and \citet{SSY21} studied non-IID regret minimization in MAB in similar models.

Bandit problems have been studied in several other models concerning communication-efficient multi-agent learning.  But those models are all different from the collaborative learning model in various aspects.  \citet{ML21} studied regret minimization in MAB in the setting where agents communicate over a general graph. \citet{WPA+20} also studied regret minimization in MAB over a general communication graph, and considered a collision scenario where the reward can be collected only when the arm is pulled by one agent at a particular time.  The analyses of algorithms in \citet{ML21} and \citet{WPA+20} treat various parameters (e.g., the number of arms $n$, the number of players $K$) as constants, and thus the bounds may not be directly comparable. \citet{SBH+13} studied gossip-based algorithms for regret minimization in MAB in P2P networks.

\citet{MHP21} studied best arm identification in MAB in a multi-agent model where the arms are partitioned into groups; each agent can only pull arms from a designated group.  One can view this model as a special case of the collaborative learning model, in which agents can pull whichever arm they would like to.  

\section{Learning with IID Data}
\label{sec:IID}

\subsection{The Algorithm}
\label{sec:algo}

The algorithm for top-$m$ arm identification in the IID data setting is described in Algorithm~\ref{alg:main}.  In the high level, our algorithm follows the  successive elimination/halving approach~\cite{ABM10,KKS13,KZZ20}. However, in order to be communication-efficient, we need to use a two-phase algorithm design. 


Let us describe Algorithm~\ref{alg:main} and the subroutines in words.  At the beginning, the coordinator picks a $(10\ln n)$-wise independent hash function $h: [n] \to [K]$ and sends it to the $K$ agents. The hash function is used to partition the $n$ arms into $K$ subsets. One choice of such a hash function is a polynomial of degree $10\ln n$ with random coefficients~\cite{WC81}; sending this function costs $O(\ln n)$ words.  

The body of Algorithm~\ref{alg:main} consists of two phases.  For convenience, we introduce the following global variables which will be used in all subroutines: 
\begin{itemize}
\item $R \triangleq \lceil \log n \rceil$: the number of rounds.

\item $n_r (0 \le r \le R)$: the number of remaining arms at the beginning of the $r$-th round.

\item $T_r (0 \le r \le R)$: the number of time steps allocated for the $r$-th round.

\item $I_r^k$: the subset of arms held by the $k$-th agent at the beginning of the $r$-th round.  

\item $Q_r^k$: the set of accepted arms by the $k$-th agent in the first $r$ rounds

\item $m_r \triangleq m - \abs{Q_r}$: the number of top-$m$ arms that are not yet accepted
\end{itemize}
In the first phase (Line~\ref{ln:a-1}-\ref{ln:a-2}), each agent is in charge of one subset of the arms.  In each round $r$, we call a subroutine \textsc{LocalElim}$(r)$ (Algorithm~\ref{alg:local}), in which each agent prunes arms in its subset locally {\em without} communicating the eliminated arms to the coordinator (and consequently other agents).  Throughout the phase, we monitor the sizes of subsets held by the $K$ agents.  
We say the $K$ subsets $\{I_r^k\}_{k \in [K]}$ of arms {\em balanced} if
\begin{equation*}
\textstyle \forall (i, j), \quad \abs{I_r^i} \in \left[\frac{1}{2}\abs{I_r^j}, 2\abs{I_r^j} \right].
\end{equation*}
Whenever the $\{I_r^k\}_{k \in [K]}$ become unbalanced, we end the first phase. The coordinator then collects the set of remaining arms  from the agents. 

In the second phase (Line~\ref{ln:a-3}-\ref{ln:a-4}), in each round $r$, we call another subroutine \textsc{GlobalElim}$(r)$ (Algorithm~\ref{alg:global}) to continue eliminating arms.  In each round, $K$ agents and the coordinator work together to eliminate a constant fraction of the arms.  The second phase ends when $r = R$. 

\begin{algorithm}[t]
	\caption{\textsc{Collab-Top-$m$-IID}$(I, K, m, T)$}
	\label{alg:main}
	\KwIn{a set $I$ of $n$ arms, $K$ agents, parameter $m$, and time horizon $T$.}
	\KwOut{top-$m$ arms with the highest means}
	
	Coordinator picks a $(10\ln n)$-wise independent hash function $h : [n] \to [K]$ and sends to each of $K$ agents\label{ln:a-6}\;
	for each $k \in [K]$, the $k$-th agent forms the set $I_0^k \gets \{i \in I \mid h(i) = k\}$\;
	let $R \gets \lceil \log n \rceil$\;
	for $r = 0, 1, \ldots, R$, set $n_r \gets \lfloor n/2^r \rfloor$\;  
	set $T_0 \gets 0$, and for $r = 1, \ldots, R$, $T_r \gets \lfloor \frac{TK2^r}{4nR}\rfloor$\;
	$r \gets 0$, $m_0 \gets m$\;
	\While{partition $\{I_r^k\}_{k \in [K]}$ is balanced \label{ln:a-1}}
	{
		$\left\{(I_{r+1}^k, Q_{r+1}^k)\right\}_{k \in [K]} \gets $\textsc{LocalElim}$(r)$\; 
		$r \gets r+1$\label{ln:a-2}\; 
	}
	for each $k \in [K]$, the $k$-th agent sends $I_r^k$, $\{\hat{\mu}^{(r)}_i\}_{i \in I_r^k}$, and $Q_r^k$ to Coordinator\label{ln:a-5}\;
	Coordinator sets $I_r\gets \bigcup_{k = 1}^{K} I_{r}^{k}$ and $Q_r \gets \bigcup_{k=1}^{K} Q_r^k$\; 
	\While{$r < R$ \label{ln:a-3}}
	{
		$(I_{r+1}, Q_{r+1}) \gets$ \textsc{GlobalElim}$(r)$\;
		$r \gets r+1$\label{ln:a-4}\;
	}
	\Return $Q_r$.
\end{algorithm}

We now explain \textsc{LocalElim} and  \textsc{GlobalElim}  in more details.

\begin{algorithm}[t]
	\caption{\textsc{LocalElim}$(r)$}
	\label{alg:local}
		For each $k \in [K]$, the $k$-th agent pulls each arm $i \in I_r^k$ for $(T_{r+1} - T_r)$ times and calculates its empirical mean $\hat{\mu}^{(r)}_i$ (after being pulled for $T_{r+1}$ times)\;
		let $\sigma_r : [n_r] \to I_r$ be a bijection such that $\hat{\mu}^{(r)}_{\sigma(1)} \ge \ldots \ge \hat{\mu}^{(r)}_{\sigma(n_r)}$\label{ln:b-2}\; 
		by calling \textsc{CollabSearch}$\left(\{\hat{\mu}^{(r)}_i\}_{i \in I_r^1}, \ldots, \{\hat{\mu}^{(r)}_i\}_{i \in I_r^K}, (n_r - m_r + 1)\right)$ and  \textsc{CollabSearch}$\left(\{\hat{\mu}^{(r)}_i\}_{i \in I_r^1}, \ldots, \{\hat{\mu}^{(r)}_i\}_{i \in I_r^K}, (n_r - m_r)\right)$, Coordinator  finds $\hat{\mu}^{(r)}_{\sigma_r(m_r)}$ and $\hat{\mu}^{(r)}_{\sigma(m_r + 1)}$ and sends them to all agents\label{ln:b-3}\;
		for each $k \in [K]$, the $k$-th agent locally computes $\hat{\Delta}^{(r)}_i = \max\left\{\hat{\mu}^{(r)}_i - \hat{\mu}^{(r)}_{\sigma_r(m_r + 1)}, \hat{\mu}^{(r)}_{\sigma_r(m_r)} - \hat{\mu}^{(r)}_i\right\}$ for each $i \in I_r^k$\label{ln:b-4}\;
		let $\pi_r : [n_r] \to I_r$ be a bijection such that $\hat{\Delta}^{(r)}_{\pi_r(1)}  \le \dotsc \le \hat{\Delta}^{(r)}_{\pi_r(n_r)}$\;
		by calling \textsc{CollabSearch}$\left(\{\hat{\Delta}^{(r)}_i\}_{i \in I_r^1}, \ldots, \{\hat{\Delta}^{(r)}_i\}_{i \in I_r^K}, n_{r+1}\right)$, Coordinator finds $\hat{\Delta}^{(r)}_{\pi_r(n_{r+1})}$ and sends it to all agents\label{ln:b-6}\;
		for each $k \in [K]$, the $k$-th agent forms the sets $E_r^k \gets \left\{i \in I_r^k \left|\ \hat{\Delta}^{(r)}_i > \hat{\Delta}^{(r)}_{\pi_r(n_{r+1})}\right.\right\}$ and $A_r^k \gets \left\{i \in E_r^k \left|\ \hat{\mu}^{(r)}_i \ge \hat{\mu}^{(r)}_{\sigma_r(m_r)}\right.\right\}$\;
		for each $k \in [K]$, the $k$-th agent sends $\abs{A_r^k}$ to Coordinator\label{ln:b-8}\;
		Coordinator sets $m_{r+1} \gets m_r - \sum_{k=1}^K \abs{A_r^k}$\;
		for each $k \in [K]$, the $k$-th agent sets $I_{r+1}^k \gets I_r^k \setminus E_r^k$, and updates $Q_{r+1}^k \gets Q_r^k \cup A_r^k$\;
		\Return $\left\{(I_{r+1}^k, Q_{r+1}^k)\right\}_{k \in [K]}$.
\end{algorithm}

\begin{algorithm}[t]
	\caption{\textsc{GlobalElim}$(r)$}
	\label{alg:global}
		 Coordinator distributes the pulls to the $K$ agents such that each arm in $I_r$ is pulled $(T_{r+1} - T_r)$ times. Concretely, let $Q \gets$ \textsc{BalancedPullDist}($I_r, T_{r+1} - T_r$); for each $(i, k, t) \in Q$, Coordinator requests agent $k$ to pull arm $i$ for $t$ times\;
		agents send the empirical mean of each arm they have pulled to Coordinator\; 
		Coordinator computes for each arm $i \in I_r$ its empirical mean $\hat{\mu}^{(r)}_i$ (after being pulled for $T_{r+1}$ times)\;
		let $\sigma_r : [n_r] \to I_r$ be a bijection such that $\hat{\mu}^{(r)}_{\sigma(1)} \ge \ldots \ge \hat{\mu}^{(r)}_{\sigma(n_r)}$\;
		Coordinator computes $\hat{\Delta}^{(r)}_i = \max\left\{\hat{\mu}^{(r)}_i - \hat{\mu}^{(r)}_{\sigma_r(m_r + 1)}, \hat{\mu}^{(r)}_{\sigma_r(m_r)} - \hat{\mu}^{(r)}_i\right\}$\;
		let $\pi_r : [n_r] \to I_r$ be a bijection such that $\hat{\Delta}^{(r)}_{\pi_r(1)} \le \dotsc \le \hat{\Delta}^{(r)}_{\pi_r(n_r)}$ \;
		Coordinator computes $E_{r} \gets \left\{i \in I_{r} \left|\ \hat{\Delta}^{(r)}_i > \hat{\Delta}^{(r)}_{\pi_r(n_{r + 1})}\right.\right\}$, $A_r \gets \left\{i \in E_r \left|\ \hat{\mu}^{(r)}_i \ge \hat{\mu}^{(r)}_{\sigma_r(m_r)}\right.\right\}$,
		and sets $I_{r+1} \gets I_r \setminus E_r$, $Q_{r+1} \gets Q_r \cup A_r$\;
		\Return $(I_{r+1}, Q_{r+1})$.
\end{algorithm}

In \textsc{LocalElim}$(r)$, each agent pulls each arm in $I_r^k$ for $(T_{r+1} - T_r)$ times, and calculates the empirical mean of the arm after $T_{r+1}$ pulls have been made on it.   After that, The coordinator uses the subroutine \textsc{CollabSearch} 
(Algorithm~\ref{alg:binary} in Appendix~\ref{sec:app-algo})
to communication-efficiently identify the arms having the $m_r$-th and $(m_r+1)$-th {\em largest} empirical means in $\bigcup_{k \in [K]} I_r^k$, and sends the two values to the $K$ agents.  Using these two empirical means, each agent $k$ is able to compute the empirical gap $\hat{\Delta}_i^{(r)}$ for each $i \in I_r^k$.  Next, the coordinator uses  \textsc{CollabSearch} again to find the $n_{r+1}$-th {\em smallest} empirical gap $\hat{\Delta}_{n_{r+1}}^{(r)}$ and sends it to all agents.  The agents then identify all arms whose empirical gaps are larger than $\hat{\Delta}_{n_{r+1}}^{(r)}$, among which they accept (by adding to the set $A_r^k$) those whose empirical means are at least the $m_r$-th largest empirical mean, and discard the others.  The key feature of \textsc{LocalElim} is that, in order to save communication, agents accept/eliminate arms locally without sending them to the coordinator.

In \textsc{GlobalElim}$(r)$, the coordinator first uses a subroutine \textsc{BalancedPullDist} 
(Algorithm~\ref{alg:partition} in Appendix~\ref{sec:app-algo})
to communication-efficiently distribute the workload evenly to the $K$ agents such that each remaining arm in $I_r$ is pulled by $(T_{r+1} - T_r)$ times.  After collecting the information from the $K$ agents, the coordinator computes the empirical mean of each arm in $I_r$ after $T_{r+1}$ pulls have been made on it.  Next, the coordinator computes the empirical gaps of all arms in $I_r$, identifies those whose empirical gaps are {\em not} among the top-$n_{r+1}$ smallest ones.  For those arms, the coordinator accepts those whose empirical means are at least the $m_r$-th largest empirical mean, and discards the rest.  The key feature of \textsc{GlobalElim} is that all the acceptance/elimination decisions are made by the coordinator.

Due to space constraints, we leave the detailed descriptions of the two auxiliary subroutines \textsc{CollabSearch} and \textsc{BalancedPullDist} to 
Appendix~\ref{sec:app-algo}.

\begin{remark}
We remark that the concept of {\em balanced partition} also appeared in the work of~\citet{WHCW20} when they studied the regret minimization problem in the collaborative learning model, but there are some notable differences between the DEMAB algorithm in \citet{WHCW20} and our Algorithm~\ref{alg:main} for top-$m$ arm identification.  First, in DEMAB, in the $r$-th round each agent locally eliminates arms whose estimated means are at least $2^{-r}$ away from that of the best arm. It is not clear how to extend this idea to the fixed-budget setting and achieve a similar success probability as (\ref{eq:success}) for best/top-$m$ arm identifications. 
Recall that in Algorithm~\ref{alg:main}, we eliminate the half of the (global) remaining arms with the lowest estimated means.
Second, in DEMAB, every time the partition becomes unbalanced, it performs a re-balancing step. While in Algorithm~\ref{alg:main},  we do not use any re-balancing step; the algorithm just enters the \textsc{GlobalElim} phase when the partition becomes unbalanced.
\end{remark}

\subsection{The Analysis}
\label{sec:analysis}

We show the following theorem for Algorithm~\ref{alg:main}.  All logarithms, unless otherwise stated, have a base of $2$.
\begin{theorem}
\label{thm:iid-ub}
	Algorithm~\ref{alg:main} returns the set of $m$ arms with highest means with probability at least 
\begin{equation}
\label{eq:success}
	1 - 2 n \log (2n) \cdot \exp\left(-\frac{TK}{128 H^{\langle m\rangle} \log (2n)}\right),
\end{equation}
	uses $T$ time steps and $O(K \log^2 n + m)$ words of communication.
\end{theorem}

Setting $T = 500 H^{\langle m\rangle}\log^2 n/K$, the success probability of $(\ref{eq:success})$ is at least $0.99$.  According to the definition of speedup (Eq.\ (\ref{eq:speedup})), we have the following corollary.
\begin{corollary}
\label{cor:full-speedup}
Algorithm~\ref{alg:main} achieves a speedup of $\Omega(K/\log^2 n)$ using  $O(K \log^2 n + m)$ words of communication.
\end{corollary}


In the rest of the section, we analyze Algorithm~\ref{alg:main}.

\paragraph{Correctness.}  
Let $\G \subseteq I$ be the subset of $m$ arms with the highest means.  We prove the correctness of Algorithm~\ref{alg:main} by induction.  Let $I_r = \bigcup_{k \in [K]} I_r^k$, $E_r = \bigcup_{k \in [K]} E_r^k$, $A_r = \bigcup_{k \in [K]} A_r^k$, and $Q_r = \bigcup_{k \in [K]} Q_r^k$.  We show that for any $r = 0, 1, \ldots, R$, 
\begin{equation}
\label{eq:a-1}
(Q_r \subseteq \G) \land (\G \subseteq Q_r \cup I_r)
\end{equation}
holds with high probability.

In the base case when $r = 0$, $(\emptyset = Q_0 \subseteq \G) \land (\G \subseteq I_0 = I)$ holds trivially.

Let $\pi : [n] \to I$ be a bijection such that $\Delta_{\pi(1)} \le \dotsc \le \Delta_{\pi(n)}$. To conduct the induction step, we introduce the following event. 
\begin{equation}
\label{event:E}
	\E : \forall{r = 0, \dotsc, R - 1}, \forall{i \in I_r} : \abs{\mu_i - \hat{\mu}^{(r)}_i} < \frac{\Delta_{\pi(n_{r + 1})}}{8}\,. \nonumber
\end{equation}
The following claim says that $\E$ happens with high probability.  Due to space constraints, we delay the proof to 
Appendix~\ref{sec:proof-cla-E}

\begin{claim}
\label{cla:E}
$$
	\Pr[\E] \ge 1 - 2 n \log(2n) \cdot \exp\left(-\frac{TK}{128 H^{\langle m\rangle} \log (2n)}\right).
$$
\end{claim}

We assume $\E$ holds in the rest of the analysis.  The following lemma implements the induction step.  Its proof is technical, and can be found in 
Appendix~\ref{sec:proof-lem-induction}.

\begin{lemma}
\label{lem:induction}
In the execution of Algorithm~\ref{alg:main}, for any $r = 0, 1, \ldots, R - 1$, if $Q_r \subseteq \G \subseteq Q_r \cup I_r$, then \textsc{LocalElim}$(r)$ (or \textsc{GlobalElim}$(r)$) returns $(I_{r+1}, Q_{r+1})$ such that $Q_{r+1} \subseteq \G \subseteq Q_{r+1} \cup I_{r+1}$.
\end{lemma}

Lemma~\ref{lem:induction}, together with the trivial base case, gives (\ref{eq:a-1}). 

Note that for $r = R$, we have $I_R = \emptyset$.  By (\ref{eq:a-1}), we have $Q_R \subseteq \G$ and $\G \subseteq Q_R$, which implies $Q_R = \G$, and thus the correctness of Algorithm~\ref{alg:main}.

\paragraph{Communication Cost.}  We analyze the communication cost of \textsc{LocalElim} and \textsc{GlobalElim} separately.  We start with \textsc{LocalElim}.  
The following lemma gives the communication cost of each call of \textsc{CollabSearch}; the proof can be found in 
Appendix~\ref{sec:proof-lem-CollabSearch}.

\begin{lemma}
\label{lem:CollabSearch}
The communication cost of \textsc{CollabSearch} 
(Algorithm~\ref{alg:binary} in Appendix~\ref{sec:app-algo})
is bounded by $O\left(K\log \abs{\bigcup_{k \in [K]} A_k}\right)$.
\end{lemma}

The next lemma bounds the communication cost of \textsc{LocalElim}.

\begin{lemma}
\label{lem:LocalElim}
The communication cost of  each call of \textsc{LocalElim} is $O(K\log n)$ words. 
\end{lemma}
\begin{proof}
By Lemma~\ref{lem:CollabSearch}, Line~\ref{ln:b-3} and Line~\ref{ln:b-6} use $O(K \log n)$ communication.  It is easy to see that Line~\ref{ln:b-8} uses $O(K)$ communication, and other steps do not need communication.  Therefore, the total communication is bounded by $O(K\log n)$.
\end{proof}

Before analyzing the communication cost of \textsc{GlobalElim}, we need the following technical lemma which states that the partition $\{I_r^k\}_{k \in [K]}$ is balanced when the set of remaining arms $I_r$ is large enough. In other words,  Algorithm~\ref{alg:main} enters the second phase (which calls \textsc{GlobalElim}) only when $\abs{I_r}$ becomes sufficiently small.  The proof of the lemma can be found in the full version of this paper.

\begin{lemma}
\label{lem:balance}

At round $r$, if $n_r \ge 100K\log n$, then the partition $\{I_r^k\}_{k \in [K]}$ is balanced with probability at least $1 - \frac{K}{n^3}$.
\end{lemma}

Now we are ready to bound the communication cost of \textsc{GlobalElim}.

\begin{lemma}
\label{lem:GlobalElim}
The communication cost of each call of \textsc{GlobalElim} is  $O(K\log n)$ words. 
\end{lemma}

\begin{proof}
By Lemma~\ref{lem:balance}, we know that when calling \textsc{GlobalElim}$(r)$, the number of arms in $I_r$ is at most $100 K\log n$.  The communication cost of the subroutine \textsc{BalancedPullDist} 
(Algorithm~5 in the full versoin of this paper) can be bounded by $O(K + \abs{I_r}) = O(K \log n)$, since the number of messages the coordinator sends to each agent $k \in [K]$ is at most $(2 + \lfloor \abs{I_r}/K \rfloor)$.  
\end{proof}

Combining Lemma~\ref{lem:LocalElim} and Lemma~\ref{lem:GlobalElim}, and noting that there are $R = O(\log n)$ rounds and the communication at Line~\ref{ln:a-6} and Line~\ref{ln:a-5} of Algorithm~\ref{alg:main} is bounded by $O(K \log n + m)$, we have 

\begin{lemma}
\label{lem:comm}
The communication cost of Algorithm~\ref{alg:main} is bounded by $O(K\log^2 n + m)$. 
\end{lemma}

\paragraph{Time Complexity.}  
We start by analyzing \textsc{LocalElim}.  Recall that we call \textsc{LocalElim}$(r)$ only when the partition $\{I_r^k\}_{k \in [K]}$ is balanced, which implies that for any $r$,
$
\max_{k \in [K]}\left\{{\abs{I_r^k}}\right\} \le {2n_r}/{K}.
$
Therefore, the number of pulls each agent makes in \textsc{LocalElim}$(r)$ is  $\abs{I_r^k}(T_{r+1}-T_r) \le \frac{2n_r}{K}(T_{r+1}-T_r)$.

In \textsc{GlobalElim}, in the $r$-th round, each agent makes $\lceil n_r(T_{r+1}-T_r)/K \rceil \le \frac{2n_r}{K}(T_{r+1}-T_r)$ pulls.

Thus, the total running time of Algorithm~\ref{alg:main} can be bounded by
\begin{equation}
	\sum_{r = 0}^{R - 1} \frac{2 n_r (T_{r + 1} - T_r)}{K} \le \sum_{r = 0}^{R - 1} \left(\frac{2 n}{ 2^r K} \cdot \frac{TK 2^{r + 1}}{4 n R}\right) \le T.
\end{equation}

\subsection{Communication Lower Bound}
\label{sec:iid-lb}

The following theorem gives a communication lower bound for collaborative top-$m$ arm identification on IID data.  

\begin{theorem}
\label{thm:iid-lb}
Any collaborative learning algorithm for top-$m$ arm identification that achieves a speedup of $\beta$ needs to use $\Omega(\beta + m)$ bits of communication.  
\end{theorem}

\begin{proof}
We first show that $\Omega(\beta)$ bits of communication is needed.  We argue that at least $\Omega(\beta)$ agents need to be involved in order to achieve a speedup of $\beta$.  Suppose that there is a collaborative algorithm for top-$m$ arm identification that involves $\beta' = o(\beta)$ agents and achieves a speedup of $\beta$, then by the definition of the speedup (Eq.\ (\ref{eq:speedup})), there is a centralized algorithm that solves the same problem with running time $\beta' \cdot H^{\langle m \rangle}(I) / \beta = o(H^{\langle m \rangle}(I))$ (e.g., by concatenating the pulls of the $\beta'$ agents), contradicting to the $\Omega(H^{\langle m \rangle})$ time lower bound in \cite{CLQ17b}.

We next show that $\Omega(m)$ is also a lower bound of the communication cost.  This is simply because we require the coordinator to output the top-$m$ arms at the end, and the coordinator cannot pull arms itself.  Note that the cost of communicating the indices of the top-$m$ arms out of the $n$ arms is $\log {n \choose m} = \Omega(m)$ (given $n \ge 2m$; see the definition of the top-$m$ arm identification problem in the introduction).  

Summing up, the total bits of communication is at least $\Omega(\max\{\beta, m\}) = \Omega(\beta + m)$.
\end{proof}

Particularly, for $\beta = \tilde{\Omega}(K)$, we have:
\begin{corollary}
\label{cor:iid-lb}
Any collaborative learning algorithm for top-$m$ arm identification that achieves a speedup of $\tilde{\Omega}(K)$ needs to use $\tilde{\Omega}(K + m)$ bits of communication.  
\end{corollary}

Combined with Corollary~\ref{cor:full-speedup}, our upper and lower bounds are tight up to logarithmic factors. 
\section{Learning with Non-IID Data}
\label{sec:non-IID}

\subsection{Algorithm and Analysis}
\label{sec:non-iid-alg}

The algorithm for top-$m$ arm identification in the non-IID data setting is very similar to \textsc{GlobalElim} (Algorithm~\ref{alg:global}), except that in the non-IID setting, the workload distribution at the beginning of each round is more straightforward: each of the $K$ agents simply pull each of the remaining arms in $I_r$ for $(T_{r+1} - T_r)/K$ times.  These information is enough for the coordinator to calculate the empirical global means of the $n$ arms for the elimination process.  Due to space constraints, we leave the full description of 
Algorithm~\ref{alg:noniid} to Appendix~\ref{sec:app-algo-noniid}.

We have the following theorem, whose proof can be found in 
Appendix~\ref{sec:proof-thm-noniid-ub}.
\begin{theorem}
\label{thm:noniid-ub}
There is an algorithm 
(Algorithm~\ref{alg:noniid} in Appendix~\ref{sec:app-algo-noniid})
 that solves top-$m$ arm identification in the non-IID data setting with probability at least  
\begin{equation}
\label{eq:success-2}
	1 - 2n \log(2n) \cdot \exp\left(-\frac{TK}{128 H^{\langle m\rangle} \log(2n)}\right)\,,
\end{equation}
uses $T$ time steps and $O(Kn \log n)$ words of communication. 
\end{theorem}

Setting $T = 500 H^{\langle m\rangle}\log^2 n/K$, the success probability of $(\ref{eq:success-2})$ is at least $0.99$. 
\begin{corollary}
\label{cor:full-speedup-2}
Algorithm~6 achieves a speedup of $\Omega(K/\log^2 n)$ using $O(Kn \log n)$ words of communication. 
\end{corollary}


\subsection{Communication Lower Bound}
\label{sec:non-iid-lb}

The following theorem states that in the non-IID data setting, even if each agent knows the local means of all the $n$ arms exactly, they still need to spend $\Omega(n K)$ bits of communication to solve top-$1$ (i.e., $m=1$) arm identification. This communication lower bound holds regardless the speedup of the collaborative algorithm.  The proof of the theorem makes use of a reduction from a problem called {\em Set-Disjointness} in multi-party communication complexity.


\begin{theorem}
\label{thm:non-iid-lb}
Any collaborative learning algorithm that solves top-$1$ arm identification with probability $0.99$ needs $\Omega(nK)$ bits of communication.
\end{theorem} 

\begin{proof}
We make use of a problem called $K$-Set-Disjointness (\DISJ) in multi-party communication complexity.  In the coordinator model, the \DISJ\ problem can be described as follows: each agent $k\ (k \in [K])$ holds a $n$-bit vector $X^k = (X_1^k, \ldots, X_n^k) \in \{0,1\}^n$, and the coordinator has no input.  The $K$ agents and the coordinator want to compute the following function via communication:
\begin{eqnarray*}
\text{\DISJ}(X^1, \ldots, X^K)   =  \left\{
  \begin{array}{rl}
   1, & \exists i \in [n] \text{ s.t. } \forall j \in [k], X_i^j = 1,\\
   0, & \text{otherwise.}
  \end{array}
  \right.
\end{eqnarray*}
It is known that any randomized algorithm that solves \DISJ\ with probability at least $0.99$ needs $\Omega(nk)$ bits of communication \cite{BEO+13}. 

We first perform an input reduction from \DISJ\ to top-$1$ arm identification.  Each agent $k$ converts its input vector (of \DISJ) $X^k = (X_1^k, \ldots, X_n^k) \in \{0,1\}^n$ to $n$ arms (of  top-$1$ arm identification) with the following means: the $i$-th arm has a local mean of 
$$(1+X_i^k)/3+i\delta,$$ 
where $\delta = 1/n^2$ is a small noise that we add to means to make sure that the best arm is unique.  
The global mean of the $i$-th ($i \in [n]$) arm is 
$$
\mu_i = \frac{1}{K} \sum_{k \in [K]}\left( \frac{1+X_i^k}{3}+i\delta \right) \in (0, 1).
$$ 
We further introduce a special arm (the $(n+1)$-th arm) with all $K$ local means being 
$$
\mu_{n+1, 1} = \ldots =  \mu_{n+1, K} = \frac{n-1}{n}+\frac{1}{2n}.
$$  
Thus $\mu_{n+1} = \frac{n-1}{n}+\frac{1}{2n}$.

Clearly, for any $i \in [n]$ such that $X_i^1 = \ldots = X_i^K = 1$, we have $\mu_i > \mu_{n+1}$. On the other hand, for any $i \in [n]$ such that there exists a $k \in [K]$ such that $X_i^k = 0$, we have $\mu_i < \mu_{n+1}$.
Therefore, if the top-$1$ arm is the $(n+1)$-th arm, then \DISJ$(X^1, \ldots, X^K) = 0$. Otherwise, if the top-$1$ arm is {\em not} the $(n+1)$-th arm, then \DISJ$(X^1, \ldots, X^K) = 1$.  Therefore, any algorithm $\A$ that solves top-$1$ arm identification on the $(n+1)$ arms with aforementioned local means can also solve the original \DISJ\ problem.  The theorem follows from the $\Omega(nK)$ lower bound of the \DISJ\ problem.
\end{proof}

\section{Experiments}
\label{sec:exp}

In this section, we present the experimental study on our algorithms.   

\begin{figure*}[t]
	\includegraphics[scale=0.27]{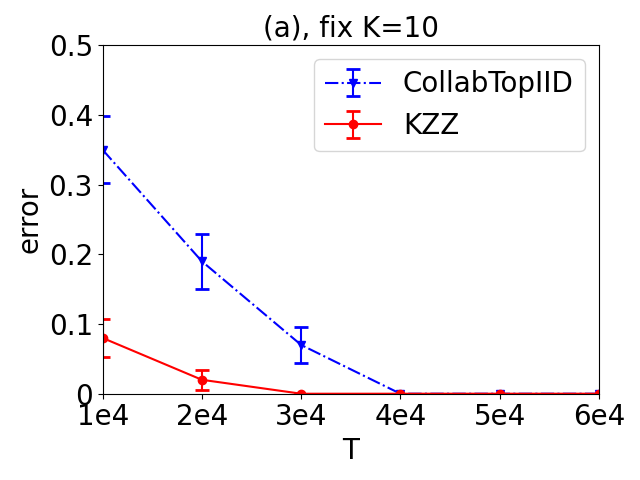}
	\includegraphics[scale=0.27]{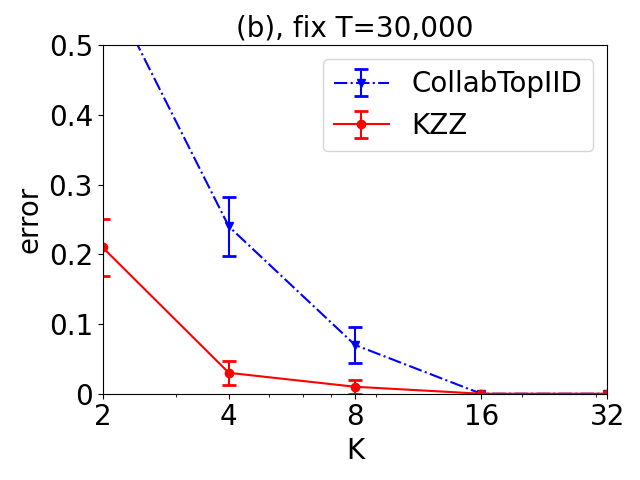}
	\includegraphics[scale=0.27]{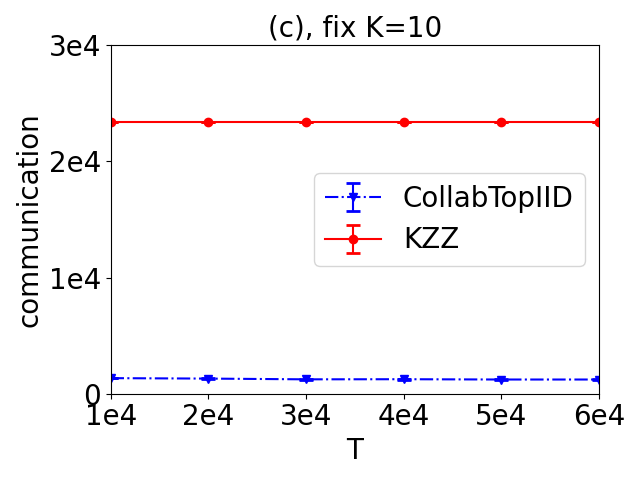}
	\includegraphics[scale=0.27]{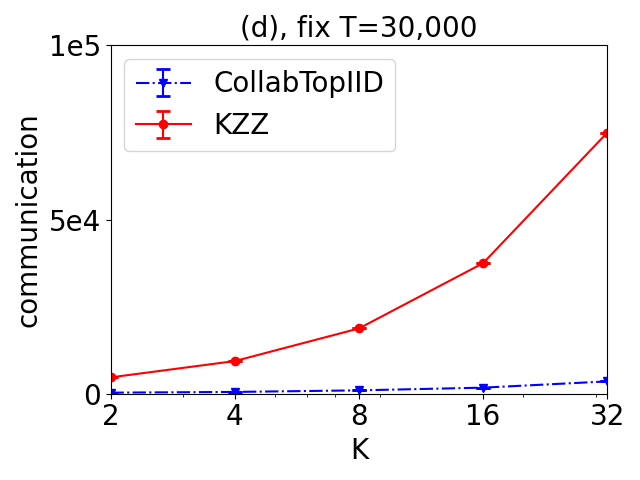}
	\caption{Performance of algorithms for top-$1$ arm identification in the IID data setting.}\label{fig:results-1}
\end{figure*}

\begin{figure*}[t]
	\includegraphics[scale=0.27]{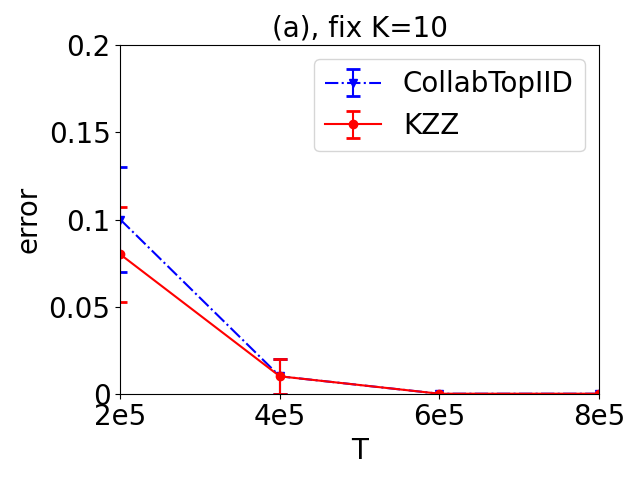}
	\includegraphics[scale=0.27]{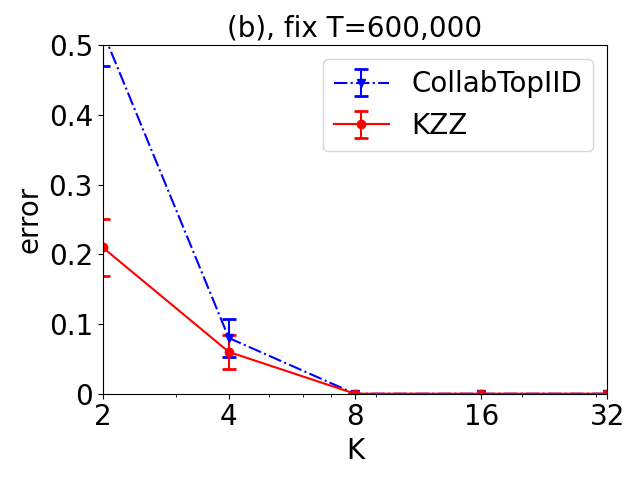}
	\includegraphics[scale=0.27]{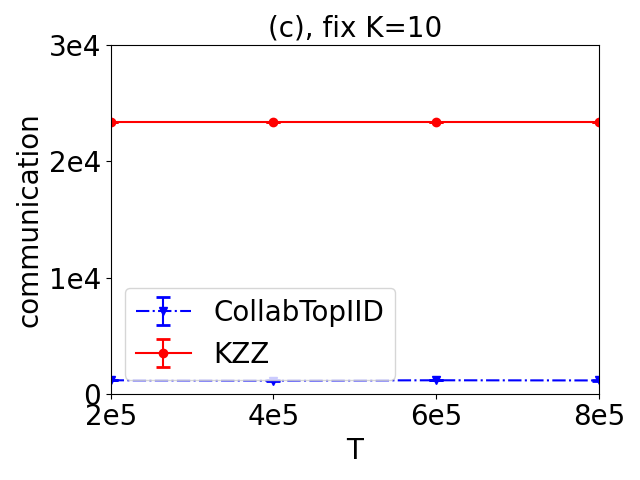}
	\includegraphics[scale=0.27]{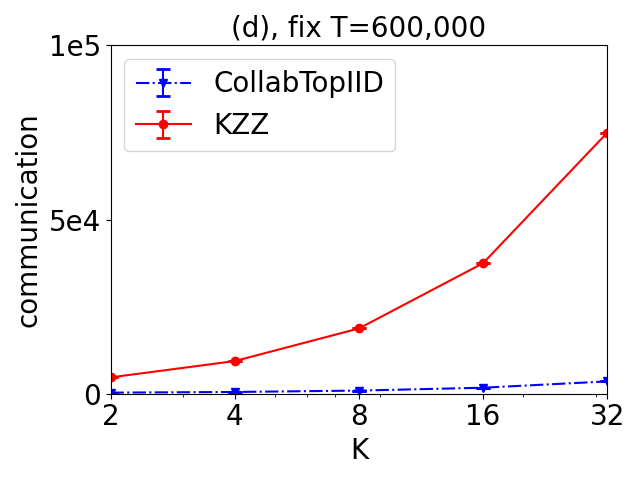}
	\caption{Performance of algorithms for top-$8$ arm identification in the IID data setting.}\label{fig:results-8}
\end{figure*}

\vspace{2mm}
\noindent{\bf Algorithms.\ \ }  In the IID data setting, we implement Algorithm~\ref{alg:main} (denoted by \CIID), and compare it with the (only) existing algorithm (denoted by \KZZ) for collaborative top-$m$ arm identification \cite{KZZ20}.  We note that the original design goal of \KZZ\ is to minimize the number of rounds of the collaborative learning process. The communication cost of \KZZ\ is $\Omega(Kn)$, which is far from being optimal. 

In the non-IID data setting, we have implemented 
Algorithm~6 (denoted by \CNIID) and tested its performance.  Since there is no direct  competitor in the non-IID data setting, we create a baseline algorithm named \Unif.  \Unif\ only uses one round. Given a time horizon $T$, each agent pulls each of the $n$ arms for $T/n$ times, and then sends the local empirical means to the coordinator.  The coordinator computes for each arm its global empirical mean, and then outputs the one with the highest global mean. Let $\Delta_{\min}$ be the mean gap between the best arm and the second best arm.  It is easy to show that when $T \ge \frac{c_T n \log (Kn)}{K\Delta_{\min}^2}$ for a sufficiently large constant $c_T$, then \Unif\ correctly outputs the best arm with probability at least $0.99$.  The communication cost of \Unif\ is $Kn$, which is the {\em lower bound} in the non-IID data setting (recall Theorem~\ref{thm:non-iid-lb}).

\vspace{2mm}
\noindent{\bf Datasets.\ \ }   We use a real-world dataset MovieLens~\cite{HK16}.  We select $588$ movies scored by at least $20,000$ users. For the $i$-th movie, we set $\mu_i$ to be the average rating of the movie divided by $5$ (to make sure that $\mu_i \in (0,1)$).  In the IID data setting, we regard each movie as an arm associated with a Bernoulli distribution of mean $\mu_i$.  

In the non-IID data setting, we split users into $K = 10$ groups (the user with ID $x$ is included in Group ($x \bmod 10$)), and select $574$ movies scored by at least $1,000$ users in each of the $K$ groups. For the $i$-th movie, we set $\mu_{i,k}$ to be the average rating of the movie in group $k$ divided by 5.

\vspace{2mm}
\noindent{\bf Computation Environment.\ \ }
All algorithms are implemented using programming language Kotlin. All experiments were conducted in PowerEdge R740 server equipped $2 \times$Intel Xeon Gold $6248$R $3.0$ GHz ($24$-core/$48$-thread per CPU) and $256$GB RAM.

\vspace{2mm}
\noindent{\bf Results.\ \ }   Our experimental results for the IID data setting are depicted in Figure~\ref{fig:results-1} (for $m=1$) and Figure~\ref{fig:results-8} (for $m=8$).  All results are averages over $100$ runs. The error bars stand for the standard deviation.

Figure~\ref{fig:results-1}(a) shows the influence of the time horizon $T$ on the error probability of outputting the best arm.  We fix the number of agents to be $10$.   We observe that the error probabilities of both \CIID\ and \KZZ\ quickly converge to zero.  \KZZ\ converges slightly faster than \CIID, which is because in order to achieve the communication efficiency, in the first phase of \CIID, the sizes of local arm sets across the $K$ agents may not be equal, which could result idle agents at  a constant fraction of time steps.  To achieve error probability $0.01$, \KZZ\ outperforms \CIID\ by a factor of 1.3 in running time.

Figure~\ref{fig:results-1}(b) shows the influence of the number of agents $K$ on the error probability of outputting the best arm.  We fix the time horizon to be $30,000$. The curves are similar to those in Figure~\ref{fig:results-1}(a).  For the same reason mentioned above, \KZZ\ converges faster than \CIID.

In Figure~\ref{fig:results-1}(c) and (d), we compare the communication cost of two collaborative algorithms \CIID\ and \KZZ.  In Figure~\ref{fig:results-1}(c) we fix $K$ and vary $T$, while in Figure~\ref{fig:results-1}(d) we fix $T$ and vary $K$.  It is clear that \CIID\ significantly outperforms \KZZ. For example, for $K = 10$ and $T = 30,000$, the communication cost of \KZZ\ is $20$ times of that of \CIID.  Recall that the communication cost of  \CIID\ is $O(K \ln^2 n + m)$, which is independent of $T$; the same holds for \KZZ.  This is why the two curves in Figure~\ref{fig:results-1}(c) are flat.  

Figure~\ref{fig:results-8} presents the results when $m$ is set to $8$ (i.e., we try to identify the top-$8$ arms).  Generally speaking, the performance trends for \CIID\ and \KZZ\ are similar as that in the case of $m = 1$.  The running time gaps between the two algorithms are smaller in Figure~\ref{fig:results-8}(a) and (b), compared with that in Figure~\ref{fig:results-1}(a) and (b).  For example, when $K$ is fixed to be $10$, to achieve error probability $0.01$, the running time of \KZZ\ and \CIID\ are the same.

Due to space constraints, we leave the experimental results on non-IID data to 
Appendix~\ref{sec:app-exp}.

%

\section*{Acknowledgments}
Nikolai Karpov and Qin Zhang are supported in part by NSF CCF-1844234 and CCF-2006591.

\bibliography{paper}

\newpage

\appendix


\noindent\rule{\textwidth/2}{1pt}
\begin{center}
	\textbf{\large Appendix for Communication-Efficient Collaborative Top Arm Identification}
\end{center}
\noindent\rule{\textwidth/2}{1pt}

\section{Missing Details in Section~\ref{sec:algo}}
\label{sec:app-algo}

Subroutines \textsc{CollabSearch} and  \textsc{BalancedPullDist} are described in Algorithm~\ref{alg:binary} and Algorithm~\ref{alg:partition}, respectively.

\begin{algorithm}
\caption{\textsc{CollabSearch}$(A_1, \dotsc, A_K, m)$}
\label{alg:binary}
\KwIn{for each $k \in [K]$, the $k$-th agent has a set of empirical means $A_k$.}
\KwOut{the $m$-th smallest element in $\bigcup_{k \in [K]} A_k$.}
\Repeat{$|A_k| = 1$ for all $k \in [K]$}{
for each $k \in [K]$, the $k$-th agent sends the median of $A_k$ (denoted by $b_k$), and $c_k = \lvert \{x \in A_k \mid x < b_k\} \rvert$\;
let $\sigma : [K] \to [K]$ be a bijection such that $b_{\sigma(1)} \le  \ldots \le b_{\sigma(k)}$\;
Coordinator finds an index $v$ such that $\sum_{i = 1}^{v} c_{\sigma(i)} < m$ and $\sum_{i = 1}^{v + 1} c_{\sigma(i)} \ge m$\;
Coordinator sends a message `$>$' to agents $\{\sigma(1), \ldots, \sigma(v)\}$ and a message `$<$' to agents $\{\sigma(v+1), \ldots, \sigma(K)\}$\;
agents who get message `$>$' removes elements smaller than $b_k$ from $A_k$, and those who get message `$<$' removes elements larger than $b_k$ from $A_k$\;
Coordinator sets $m \gets m - \sum_{i = 1}^{v} c_{\sigma(i)}$\;
}
for each $k \in [K]$, the $k$-th agent sends Coordinator the single element  in $A_k$; Coordinator finds the $m$-th smallest element $b$ among these $K$ elements\label{ln:d-8}\;
\Return $b$.
\end{algorithm}

\begin{algorithm}
	\caption{$\textsc{BalancedPullDist}(I, B)$}
	\label{alg:partition}
	\KwIn{set of arms $I = \{I_1, \ldots, I_{\abs{I}}\}$, number of pulls for each arm.}
	\KwOut{output a set of tuples $\{(i, k, t)\}$, where a tuple $(i, k, t$) means that agent $k$ should pull arm $I_i$ for $t$ times} 
	for each $i \in [n]$ set $T_i \gets B$\;
	set $n \gets \abs{I}$, $Q \gets \emptyset$\;
	set $i \gets 1$\;
	\For{$k = 1, \dotsc, K$}
	{
		set $r \gets \lceil \abs{I} B/K \rceil$\;
		\While{$i \le n$ and $r > 0$}
		{
			add $(i, k,  \min(T_i, r))$ to $Q$\;
			$T_i \gets T_i - \min(T_i, r)$, $r \gets r - \min(T_i, r)$\;
			\lIf{$T_i = 0$}
			{
				$ i \gets i + 1$
			}
		}
	}
	\Return{$Q$}.
\end{algorithm}

We briefly describe the two algorithms in words.  In \textsc{CollabSearch}$(A_1, \ldots, A_k, m)$, $K$ agents and the coordinator use a binary-search like strategy to compute the $m$-th smallest element in the union of $K$ sets $\bigcup_{k \in [K]}A_k$ in a communication-efficient way.  In each iteration, each agent $k$ sends the median $b_k$ of its local set to the coordinator and the number of elements (denoted by $c_k$) that are smaller than the median.  Based on $\{b_k, c_k\}_{k \in [K]}$, the coordinator informs each agent which half of arms to remove from its local set. 

In \textsc{BalancedPullDist}$(I, B)$, the coordinator assigns the set of arms $I$ to $K$ agents such that each arm is pulled $B$ times. The assignment makes sure that each agent makes the same number of pulls.  To make the subroutine communication-efficient, the coordinator conducts a linear scan of the arms, and assigns them to agents in a greedy way.  More precisely, the coordinator first computes the number of pulls each agent should perform, and then tries to assign each arm to a single agent whenever possible, or to a minimal number of agents otherwise.

\section{Missing Details in Section~\ref{sec:analysis}}
\label{sec:app-analysis}

\subsection{Proof of Lemma~\ref{cla:E} }
\label{sec:proof-cla-E}

Recall that in both Algorithm~\ref{alg:local} and Algorithm~\ref{alg:global}, $\hat{\mu}_i^{(r)}$ is the average of $T_{r+1}$ i.i.d.\ random variables sampled from $\D_i$ with mean $\mu_i$.  By Hoeffding's inequality (Lemma~\ref{lem:chernoff}), and using the inequality $\max_{i \in I} \left\{{i}/{\Delta^2_{\pi(i)}}\right\} \le H^{\langle m\rangle}$ \cite{ABM10}, we have
\begin{eqnarray}
&& \Pr\left[\abs{\hat{\mu}^{(r)}_i - \mu_i} \ge \frac{\Delta_{\pi(n_{r+1})}}{8} \right] \nonumber \\
&\le& 2 \exp\left(-\frac{\Delta^2_{\pi(n_{r+1})}T_{r+1}}{32} \right)\nonumber \\
&\le& 2 \exp\left(- \frac{\Delta^2_{\pi(n_{r+1})} TK}{128 n_{r+1} R}\right)\nonumber \\
	&\le& 2\exp\left(- \frac{TK}{128 H^{\langle m \rangle} R}\right). \label{eq:c-1} \nonumber
\end{eqnarray}
By a union bound over $i = 1, \ldots, n$ and $r = 0, \ldots, R-1$, we have
\begin{eqnarray}
		\Pr\left[\bar{\E}\right] &\le& \sum_{r = 0}^{R - 1} \sum_{i = 1}^n\Pr\left[\abs{\hat{\mu}^{(r)}_i - \mu_i} \ge \frac{\Delta_{\pi(n_{r + 1})}}{8} \right] \nonumber \\ 
		&\le& 2 n R \exp\left(- \frac{TK}{128 H^{\langle m \rangle} R}\right)\,. \label{eq:c-2}
	\end{eqnarray}
The claim follows from (\ref{eq:c-2}) and the fact that $R \le \log n+1 = \log(2n)$.

\subsection{Proof of Lemma~\ref{lem:induction}}
\label{sec:proof-lem-induction}

We prove for \textsc{LocalElim}$(r)$; the proof for \textsc{GlobalElim}$(r)$ is essentially the same. 

Let $\rho_r : [n_r] \to I_r$ be the bijection such that $\mu_{\rho_r(1)} \ge \dotsc \ge \mu_{\rho_r(n_r)}$.  

By $\E$, we have for any $i = 1, \ldots, n_r$
\begin{equation}
	\label{eq:d-1}
		\mu_i - \frac{\Delta_{\pi(n_{r+1})}}{8} \le \hat{\mu}^{(r)}_i	\le \mu_i + \frac{\Delta_{\pi(n_{r+1})}}{8}\,.
\end{equation}
In particular, for any $i \in \{\rho(1), \ldots, \rho(m_r)\}$, it holds that $\hat{\mu}^{(r)}_i \ge {\mu}_{\rho(m_r)} - \frac{\Delta_{\pi(n_{r+1})}}{8}$, which, combined with the definition of $\sigma_r$ in \textsc{LocalElim}$(r)$ (Line~\ref{ln:b-2}), gives 
\begin{equation}
\label{eq:d-4}
\hat{\mu}_{\sigma_r(m_r)}^{(r)} \ge {\mu}_{\rho(m_r)} - \frac{\Delta_{\pi(n_{r+1})}}{8}.
\end{equation}
By a symmetric argument, we have 
\begin{equation}
\label{eq:d-5}
\hat{\mu}_{\sigma_r(m_r+1)}^{(r)} \le {\mu}_{\rho(m_r+1)} + \frac{\Delta_{\pi(n_{r+1})}}{8}.
\end{equation}
Recall the definition of $\Delta_i$ (Eq.\ (\ref{eq:gap})):
\begin{equation}
\label{eq:d-8}
\Delta_i = \max\left\{\mu_i - \mu_{\rho(m_r+1)}, \mu_{\rho(m_r)} - \mu_i\right\},
\end{equation}
and the definition of  $\Delta_i^{(r)}$ (Line~\ref{ln:b-3} of \textsc{LocalElim}$(r)$):
\begin{equation}
\label{eq:d-9}
\hat{\Delta}^{(r)}_i = \max\left\{\hat{\mu}^{(r)}_i - \hat{\mu}^{(r)}_{\sigma_r(m_r + 1)}, \hat{\mu}^{(r)}_{\sigma_r(m_r)} - \hat{\mu}^{(r)}_i\right\}.
\end{equation}

Combining (\ref{eq:d-1}), (\ref{eq:d-4}),  (\ref{eq:d-5}), (\ref{eq:d-8}), and (\ref{eq:d-9}), we have
\begin{equation}
\label{eq:d-2}
		\Delta_i -  \frac{\Delta_{\pi(n_{r+1})}}{4} \le \hat{\Delta}^{(r)}_{i} \le \Delta_i + \frac{\Delta_{\pi(n_{r+1})}}{4}.
\end{equation}
By (\ref{eq:d-2}), for any $i$ such that $\Delta_i \ge \Delta_{\pi(n_{r+1})}$, we have
\begin{equation}
\label{eq:d-3}
		\hat{\Delta}^{(r)}_{i} \ge \Delta_{\pi(n_{r+1})} - \frac{\Delta_{\pi(n_{r+1})}}{4}  = \frac{3{\Delta}_{\pi(n_{r+1})}}{4}.
\end{equation}

Recall that in \textsc{LocalElim}$(r)$, we add to $E_r$ those arm $i$ for which $\hat{\Delta}^{(r)}_{i}$ is larger than $\hat{\Delta}^{(r)}_{\pi_r(n_{r+1})}$. By (\ref{eq:d-3}), we have
\begin{equation}
\label{eq:d-7}
	\forall {i \in E_r}, \quad \hat{\Delta}^{(r)}_{i} \ge \hat{\Delta}^{(r)}_{\pi_r(n_{r+1})} \ge \frac{3{\Delta}_{\pi(n_{r+1})}}{4}.
\end{equation}
We analyze arms $i \in E_r$ in two cases.

{\bf Case I: $i \in E_r \backslash A_r$.} In this case, $\hat{\Delta}^{(r)}_i  = \hat{\mu}^{(r)}_{\sigma_r(m_r)} - \hat{\mu}^{(r)}_i$,
plugging which to (\ref{eq:d-7}), we have
\begin{equation}
\label{eq:e-1}
		 \hat{\mu}^{(r)}_i \le \hat{\mu}^{(r)}_{\sigma_r(m_r)} - \frac{3 \Delta_{\pi(n_{r+1})}}{4}.
\end{equation}

On the other hand, for any $i \in I_r \cap \G$, we have
\begin{eqnarray}
	 	\hat{\mu}^{(r)}_i &\ge& \mu_i - \frac{\Delta_{\pi(n_{r+1})}}{8} \nonumber \\
	 	&\ge& \mu_{\rho_r(m_r)} - \frac{\Delta_{\pi(n_{r+1})}}{8} \nonumber \\
	 	 &\ge&  \hat{\mu}^{(r)}_{\sigma_r(m_r)} - \frac{\Delta_{\pi(n_{r+1})}}{4}, \label{eq:e-2}
\end{eqnarray}
where in the first inequality we have used (\ref{eq:d-1}). The second inequality is due to the fact $i \in \G$ and the induction hypothesis $\G \subseteq Q_r \cup I_r$ (note that $\abs{Q_r} = m - m_r$). The third inequality follows from (\ref{eq:d-4}).

By (\ref{eq:e-1}) and (\ref{eq:e-2}), we know that 
\begin{equation*}
(E_r\backslash A_r) \cap (I_r \cap \G) = \emptyset,
\end{equation*}
which implies $I_r \cap \G \subseteq A_r \cup (I_r \backslash E_r) = A_r \cup I_{r+1}$, and consequently 
\begin{equation}
\label{eq:e-4}
\G \subseteq Q_{r+1} \cup I_{r+1}.
\end{equation}

{\bf Case II: $i \in A_r$.} In this case, $\hat{\Delta}^{(r)}_i  = \hat{\mu}^{(r)}_i - \hat{\mu}^{(r)}_{\sigma_r(m_r+1)}$, plugging which to (\ref{eq:d-7}), we have
\begin{equation}
\label{eq:e-5}
		 \hat{\mu}^{(r)}_i \ge \hat{\mu}^{(r)}_{\sigma_r(m_r+1)} + \frac{3 \Delta_{\pi(n_{r+1})}}{4}.
\end{equation}

On the other hand, for any $i \in I_r \backslash \G$, we have
\begin{eqnarray}
	\hat{\mu}^{(r)}_i &\le& \mu_i + \frac{\Delta_{\pi(n_{r+1})}}{8} \nonumber \\
	 &\le& \mu_{\rho_r(m_r+1)} + \frac{\Delta_{\pi(n_{r+1})}}{8} \nonumber \\
	 &\le& \hat{\mu}^{(r)}_{\sigma_r(m_r+1)} + \frac{\Delta_{\pi(n_{r +1})}}{4}, \label{eq:e-6}
\end{eqnarray}
where in the first inequality we have used (\ref{eq:d-1}). The second inequality is due to the fact $i \in I_r \backslash \G$ and the induction hypothesis $Q_r \subseteq \G$ (note that $\abs{Q_r} = m - m_r$). The third inequality is due to (\ref{eq:d-5}).

By (\ref{eq:e-5}) and (\ref{eq:e-6}), we know that
\begin{equation*}
A_r \cap (I_r \backslash \G) = \emptyset,
\end{equation*}
which implies $A_r \subseteq \G$. Consequently
\begin{equation}
\label{eq:e-7}
Q_{r+1} \subseteq \G.
\end{equation}
The lemma follows from (\ref{eq:e-4}) and (\ref{eq:e-7}).

\subsection{Proof of Lemma~\ref{lem:CollabSearch}}
\label{sec:proof-lem-CollabSearch}

Let $z = \abs{\bigcup_{k \in [K]} A_k}$.  First, it is easy to see that the main loop of \textsc{CollabSearch} runs for at most $O(\log z)$ iterations, since at each round,  half of the elements have been removed in each list.  In each round, $O(K)$ words have been exchanged between the parties.  After the loop (Line~\ref{ln:d-8}), the communication cost is bounded by $O(K)$.  Therefore, the total communication is $O(K \log z)$. 

\subsection{Proof of Lemma~\ref{lem:balance}}
\label{sec:proof-lem-balance}

We first observe that the set $I_r$ does not depend on the random hash function $h$, as $I_r = \bigcup_{k \in [K]}I_r^k$ is the {\em union} of all remaining arms held by the $K$ agents at the beginning of the $r$-th round.

For each $k \in [K]$ and each $i \in I_r$, let $X_i^k$ be the indicator random variable of the event that $h(i) = k$.  We thus have $\abs{I_r^k} = \sum_{i \in I_r} X_i^k$.  Since $\bE[X_i^k] = \frac{1}{K}$, we have 
\begin{equation}
\bE\left[\abs{I_r^k}\right] = \bE\left[\sum_{i \in I_r} X_i^k\right] = \frac{n_r}{K}.
\end{equation}

Using a variant of Chernoff-Hoeffding inequality (Lemma~\ref{lem:ch-k}), setting $\delta = \sqrt{\frac{10K\log n}{n_r}}$, and noting that $\lfloor \delta^2 \mu e^{-1} \rfloor \le 10\log n$, we have
\begin{eqnarray}
		\Pr\left[\abs{\abs{I_r^k} - \frac{n_r}{K}} \ge \sqrt{\frac{10 n_r \log n}{K}}\right] 
		&\le& \exp(-\lfloor\delta^2 \mu / 3\rfloor) \nonumber\\
		 &\le& \exp(-3 \log n) \nonumber\\
		 &\le& n^{-3}\,. \nonumber
\end{eqnarray}
By a union bound over $k = 1, \ldots, K$, with probability at least $1 - K/n^3$, we have 
\begin{equation}
\label{eq:b-2}
\forall k \in [K],\quad \abs{\abs{I_r^k} - \frac{n_r}{K}} \le \sqrt{\frac{10 n_r \log n}{K}}.
\end{equation}
When $n_r \ge 100K\log n$, we have $\frac{n_r}{K} \ge 3\sqrt{\frac{10n_r\log n}{K}}$. Using (\ref{eq:b-2}) we have
\begin{equation}
\forall a, b \in [K],\quad \abs{I_r^a} \le 2\abs{I_r^b}.
\end{equation}
Therefore, the partition $\{I_r^k\}_{k \in [K]}$ is balanced according to the definition.

\section{Missing Details in Section~\ref{sec:non-IID}}

\subsection{The Algorithm Description}
\label{sec:app-algo-noniid}
The algorithm for finding the top-$m$ arms in the non-IID data setting is described in Algorithm~\ref{alg:noniid}.

\begin{algorithm}
	\caption{$\textsc{Collab-Top-$m$-NonIID}(I, m, T)$}
	\label{alg:noniid}
	\KwIn{a set $I$ of $n$ arms, parameter $m$, and time horizon $T$.}
	\KwOut{set of arms with $m$ highest means.}
	
	let $R \gets \lceil \log n\rceil$\;
	for $r = 0, \dotsc, R$, set $n_r \gets \lfloor n / 2^r \rfloor$\;
	set $T_0 \gets 0 $, and for $r = 1, \dotsc, R$, set $T_r \gets \lfloor \frac{TK 2^r}{2 n R}\rfloor$\;
	set $I_0 \gets I$, and $m_0 \gets m$\;
	set $Q_0 \gets \emptyset$\;
	\For{$r = 0, \dotsc, R - 1$}{
		for each $k \in [K]$, agent $k$ pulls each arm $i \in I_r$ for $(T_{r + 1} - T_r)/K$ times, computes estimated mean $\hat{\mu}^{(r)}_{i,k}$ after $T_{r + 1}/K$ pulls, and sends it to Coordinator\;
		Coordinator computes $\hat{\mu}^{(r)}_i \gets \frac{1}{K} \sum_{k \in [K]} \hat{\mu}^{(r)}_{i, k}$\;
		let $\sigma_r : [n_r] \to I_r$ be a bijection such that $\hat{\mu}^{(r)}_{\sigma_r(i)} \ge \ldots \ge \hat{\mu}^{(r)}_{\sigma_r(n_r)}$\;
		for each $i \in I_r$, Coordinator computes $\hat{\Delta}^{(r)}_i \gets \max\left\{\hat{\mu}^{(r)}_i - \hat{\mu}^{(r)}_{\sigma_r(m_r + 1)}, \hat{\mu}^{(r)}_{\sigma_r(m_r)} - \hat{\mu}^{(r)}_i\right\}$\;
		let $\pi_r : [n_r] \to I_r$ be a bijection such that $\hat{\Delta}^{(r)}_{\pi_r(1)} \le \dotsc \le \hat{\Delta}^{(r)}_{\pi_r(n_r)} $\;
		Coordinator computes $E_r \gets \left\{i \in I_r \mid \hat{\Delta}^{(r)}_i > \hat{\Delta}^{(r)}_{\pi_r(n_{r + 1})}\right\}$, $A_r \gets \left\{i \in E_r \mid \hat{\mu}^{(r)}_i \ge \hat{\mu}^{(r)}_{\sigma_r(m_r)}\right\}$, and sets $I_{r + 1} \gets I_r \setminus E_r$ and $Q_{r + 1} \gets Q_r \cup A_r$\;
	}
	\Return{$Q_{R}$}
\end{algorithm}

\subsection{Proof of Theorem~\ref{thm:noniid-ub}}
\label{sec:proof-thm-noniid-ub}
\paragraph{Correctness.}
The proof of the correctness of Algorithm~\ref{alg:noniid} is very similar to that of Theorem~\ref{thm:iid-ub}.  Note that in the IID data setting (Algorithm~\ref{alg:main}), each arm is pulled for $T_{r+1}$ times in total by the $K$ agents after the $r$-th round. While in the non-IID setting (Algorithm~\ref{alg:noniid}), each arm is pulled for $T_{r+1}/K$ by {\em each} agent.  Since the global mean is the average of $K$ local means, intuitively, we can think the $K$ agents collectively pull each arm $i$ from its global mean distribution for $T_{r+1}/K \cdot K = T_{r+1}$ times.  This is why the proof of correctness for Theorem~\ref{thm:noniid-ub} is essentially the same as that for Theorem~\ref{thm:iid-ub}.  We include the proof for completeness.

Let $\G\subseteq I$ be the subset of $m$ arms with the highest global means $\mu_{i} = \frac{1}{K} \sum_{k = 1}^K \mu_{i, k}$. We show that for any $r = 0, 1, \dotsc, R$,
\begin{equation}\label{eq:induction}
	(Q_r \subseteq \G) \land (\G \subseteq Q_r \cup I_r)
\end{equation}
holds with a high probability.

In the base case when $r=0$, $(\emptyset = Q_0 \subseteq \G) \land (\G \subset I_0 = I)$ holds trivially.

For the induction step, we introduce the following event. Let $\pi : [n] \to I$ be a bijection such that $\Delta_{\pi(1)} \le \dotsc \le \Delta_{\pi(n)}$. Define 
\begin{equation}\label{eq:e1}
	\E' : \forall{r = 0, \dotsc, R-1}, \forall{i \in I_r} : \abs{\mu_i - \hat{\mu}^{(r)}_i} < \frac{\Delta_{\pi(n_{r + 1})}}{8}\,.
\end{equation}

The following claim says that $\E'$ happens with high probability. 

\begin{claim}
	\[\Pr[\E'] \ge 1 - 2n \log(2n) \cdot \exp\left(-\frac{TK}{128 H^{\langle m\rangle} n \log(2n)}\right)\,.\]
\end{claim}

\begin{proof}
	Recall that $\hat{\mu}^{(r)}_i = \frac{1}{K} \sum_{k \in [K]} \hat{\mu}^{(r)}_{i,k}$ where $\hat{\mu}^{(r)}_{i,k}$ the average of $\eta \triangleq \frac{T_{r+1}}{K}$ i.i.d. random variables $X_{i,k,1}, \ldots, X_{i,k,\eta}$, each of which is sampled from distribution $\mathcal{D}_{i,k}$ with mean $\mu_{i,k}$.  We thus have
$$
\hat{\mu}^{(r)}_i  = \frac{1}{K\eta} \sum_{k \in [K]}\sum_{j \in [\eta]} X_{i,k,j} = \frac{1}{T_{r+1}} \sum_{k \in [K]}\sum_{j \in [\eta]} X_{i,k,j},$$
and 
$$\bE[\hat{\mu}^{(r)}_i] = \mu_i.
$$
By Hoeffding's inequality (Lemma~\ref{lem:chernoff}), and using the inequality $\max_{i \in I}\{i / \Delta^2_{\pi(i)}\} \le H^{\langle m\rangle}$~\cite{ABM10}, we have
	\begin{eqnarray*}
	\Pr\left[\hat{\mu}^{(r)}_i - \mu^{(r)}_i \ge \frac{\Delta_{\pi(n_{r + 1})}}{8}\right] &\le&  2 \exp\left(-\frac{\Delta^2_{\pi(n_{r+1})} T_{r+1}}{32}\right) \\
	&\le& 2 \exp\left(-\frac{TK}{128H^{\langle m\rangle} R}\right).
	\end{eqnarray*}	
	
	By a union bound over all $i = 1, 2, \ldots, n$ and $r = 0, 1, \ldots, R - 1$, we have
	\begin{eqnarray*}
		\Pr[\bar{\E'}] &\le& \sum_{r = 0}^{R - 1} \sum_{i = 1}^n \Pr\left[\abs{\hat{\mu}^{(r)}_i - \mu_i} \ge \frac{\Delta_{\pi(n_{r + 1})}}{8}\right] \\
		&\le& 2nR \exp\left(-\frac{TK}{128 H^{\langle m \rangle} R}\right)\,.
	\end{eqnarray*}
\end{proof}

We assume $\E'$ holds in the rest of the analysis. The following lemma implements the induction step.  The proof of Lemma~\ref{lem:noniid-correct} is essentially the same to the proof of Lemma~\ref{lem:induction}, and is omitted here.

\begin{lemma}
\label{lem:noniid-correct}
In the execution  of Algorithm~\ref{alg:noniid}, for any $r = 0, 1, \ldots, R - 1$, if $Q_r \subseteq \G \subseteq Q_r \cup I_r$, then $Q_{r + 1} \subseteq \G \subseteq Q_{r + 1} \cup I_{r + 1}$.
\end{lemma}

For $r=R$, we have $I_R = \emptyset$.  Consequently, \eqref{eq:induction} implies $Q_R = \G$.

\paragraph{Communication Cost.} The number of rounds in Algorithm~\ref{alg:noniid} is bounded by $R \le \log(2n)$. In each round, the coordinator sends $O(\abs{I_r}K) = O(nK)$ words to each agent indicating which arms should be pulled. Each of the $K$ agents send $O(\abs{I_r}) = O(n)$ words about $\{\hat{\mu}^{(r)}_{i,k}\}_{i \in I_r}$ back to the coordinator.  Therefore, the total communication cost is bounded by $O(nK \cdot R) = O(nK \log n)$.

\section{Missing Details in Section~\ref{sec:exp}}
\label{sec:app-exp}

The experimental results for the non-IID data setting are depicted in Figure~\ref{fig:results-noniid-1} (for $m = 1$) and Figure~\ref{fig:results-noniid-8} (for $m = 8$).

\begin{figure}
	\includegraphics[scale=0.25]{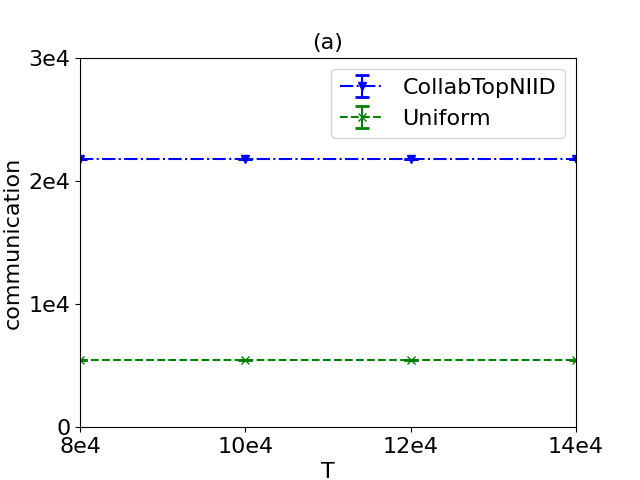}
	\includegraphics[scale=0.25]{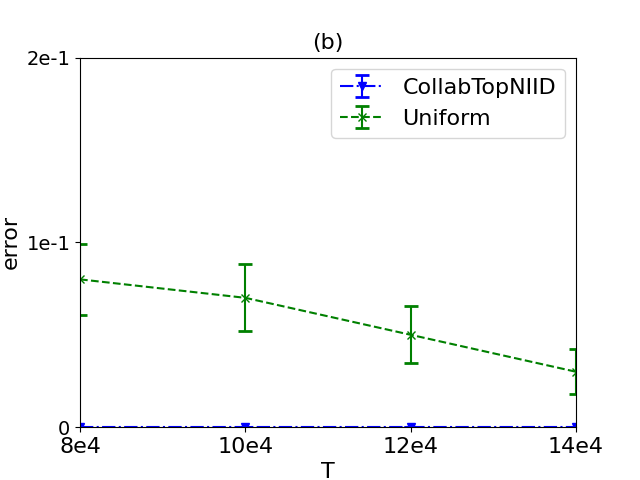}
	\caption{Performance of algorithms for top-$1$ arm identification in the non-IID setting.}\label{fig:results-noniid-1}
\end{figure}

\begin{figure}
	\includegraphics[scale=0.25]{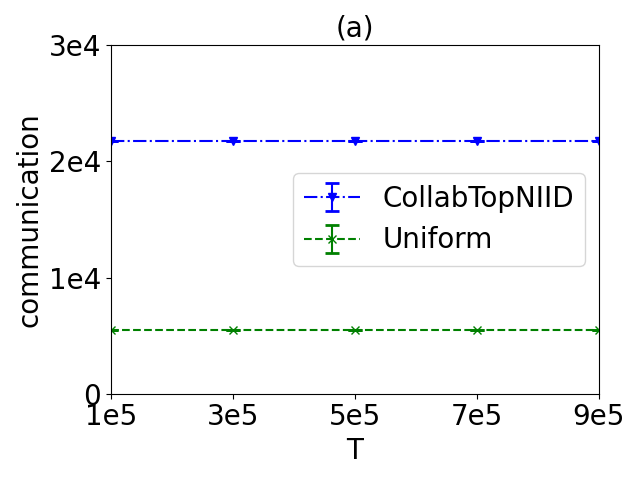}
	\includegraphics[scale=0.25]{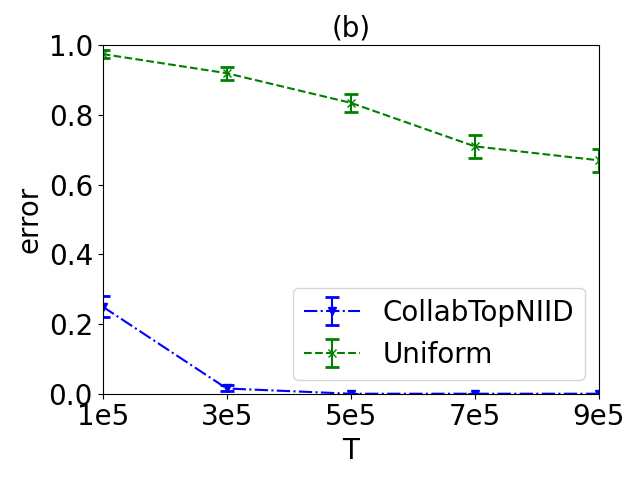}
	\caption{Performance of algorithms for top-$8$ arm identification in the non-IID setting.}\label{fig:results-noniid-8}
\end{figure}

Recall that the communication cost of \Unif\ is the lower bound of all collaborative algorithms in the non-IID setting.  From Figure~\ref{fig:results-noniid-1}(a) and Figure~\ref{fig:results-noniid-8}(a) we observe that the communication cost of \CNIID\ is not very far from that of \Unif\ (about a factor of $4$).  However, in Figure~\ref{fig:results-noniid-1}(b) and Figure~\ref{fig:results-noniid-8}(b), the error probability of \CNIID\ is much smaller than that of \Unif.

\section{Tools}
\label{sec:tool}

\begin{lemma}[Hoeffding's inequality]\label{lem:chernoff}
	Let $X_1, \dotsc, X_n \in [0, 1]$ be independent random variables and $X = \sum_{i = 1}^n X_i$. Then 
	\[
	\Pr[X > \bE[X] + t] \le \exp(-2t^2 / n)
	\]
	and
	\[
	\Pr[X < \bE[X] + t] \le \exp(-2t^2 / n)\,.
	\]
\end{lemma}

\begin{lemma}\cite{SSS93}\label{lem:ch-k}
	Let $X_1, \dotsc, X_n \in [0, 1]$ be a $t$-wise independent random variables and $\mu = \sum_{i = 1}^n\bE[X_i]$. Then for $\delta \le 1$ holds:
	\begin{itemize}
		\item if $t \le \lfloor \delta^2 \mu e^{-1/3}\rfloor$, $\Pr\left[\lvert \sum\limits_{i = 1}^n X_i - \mu\rvert \ge \delta \mu\right] \le e^{- \lfloor t / 2\rfloor}$.
		\item if $t \ge \lfloor \delta^2 \mu e^{-1/3}\rfloor$, $\Pr\left[\lvert \sum\limits_{i = 1}^n X_i - \mu\rvert \ge \delta \mu\right] \le e^{- \lfloor \delta^2 \mu / 3\rfloor}$.
	\end{itemize}
\end{lemma}

\end{document}